\documentclass{article}
\usepackage{iclr2027_conference,times}

\usepackage{amsmath,amssymb,amsthm,mathtools}
\usepackage{enumitem}
\usepackage{hyperref}
\usepackage{cleveref}
\usepackage{tcolorbox}
\usepackage{tikz}
\usetikzlibrary{automata, positioning, calc}
\usepackage{tikz-cd}
\usepackage{stmaryrd}
\usepackage{braket}
\usepackage{algorithm}
\usepackage{algpseudocode}
\usepackage{xurl}
\usepackage{caption}

\usepackage{pgfplots}
\pgfplotsset{compat=1.18}
\usepgfplotslibrary{groupplots} 

\theoremstyle{definition}
\newtheorem{definition}{Definition}
\newtheorem{example}[definition]{Example}

\theoremstyle{plain}
\newtheorem{theorem}[definition]{Theorem}

\newtheorem{corollary}[definition]{Corollary}

\theoremstyle{remark}
\newtheorem{remark}[definition]{Remark}

\newtcolorbox{todo}{
  colback=orange!10,
  colframe=orange!80!black,
  title={\textbf{TODO}},
  fonttitle=\bfseries
}

\newtcolorbox{commentary}{
  colback=blue!5,
  colframe=blue!60!black,
  title={\textbf{Commentary}},
  fonttitle=\bfseries
}

\newcommand{\F}{\mathcal{F}}

\renewcommand{\d}{\mathrm{d}}
\newcommand{\dx}{\d x}
\newcommand{\dy}{\d y}
\newcommand{\dz}{\d z}

\newcommand{\ds}{\d s}

\newcommand{\R}{\mathbb{R}}

\newcommand{\E}{\mathbb{E}}

\newcommand{\Bern}{\mathbf{Bern}}

\newcommand{\Exp}{\mathbf{Exp}}

\newcommand{\SparseLU}{{\color{blue}\textsf{\small SparseLU}}}
\newcommand{\BiCGSTAB}{{\color{blue}\textsf{\small BiCGSTAB}}}
\newcommand{\LogDerivative}{{\color{red}\textsf{\small LogDerivative}}}
\newcommand{\UniformPath}{{\color{olive}\textsf{\small UniformPath}}}
\newcommand{\Jacobi}{{\color{teal}\textsf{\small Jacobi}}}

\newcommand{\mupath}{\mu^{\text{path}}}
\newcommand{\muout}{\mu^{\text{out}}}

\newcommand{\prob}{\mathsf{p}}

\renewcommand{\epsilon}{\varepsilon}

\newcommand{\pMCsimpleexample}{
  \begin{tikzpicture}[->, node distance=2.5cm, auto]
    \node[state, initial, initial text=] (i) {$i$};
    \node[state, below of=i] (j) {$j$};
    \node[right of=j] (good) {$+1$};
    \node[right of=i] (bad) {$-1$};

    \path 
      (i) edge [loop, out=115, in=65, distance=.8cm] node {$1-p$} (i)
      (i) edge [] node {$\frac{p}{2}$} (bad)
      (i) edge [left] node {$\frac{p}{2}$} (j)
      (j) edge [] node {$1-p$} (good)
      (j) edge [] node {$p$} (bad)
    ;
  \end{tikzpicture}
}

\newcommand{\pMCmixedexample}{
  \begin{tikzpicture}[->, node distance=3cm, auto]
    \node[state, initial, initial text=] (i) {$i$};
    \node[state, below of=i] (c) {$c$};
    \node[state, right of=c, label=below:{$\vdots$},
          label=above:{$\vdots$}, label=315:{$x\in  \R_+$}
                     ] (x) {$x$};
    \node[state, right of=i] (d) {$d$};
    \node[right of=x] (yes) {$+1$};
    \node[right of=d] (no) {$-1$};

    \path (i) edge [] node {$\frac{1}{2}$} (c);
    \path (i) edge [bend left] node {$\frac{1}{2}$} (d);
    \path (d) edge [bend left] node {$1-\sqrt{p}$} (i);
    \path (c) edge [%dashed
                    ] node {$pe^{-px}$} (x);
    \path (x) edge [above right, pos = .4] node {$1-e^{-\frac{x}{2}}$} (i);
    \path (d) edge [] node {$\sqrt{p}$} (no);
    \path (x) edge [] node {$e^{-\frac{x}{2}}$} (yes);
  \end{tikzpicture}
}

\newcommand{\MCsimpleexample}{
  \begin{tikzpicture}[->, node distance=2.5cm, auto]
    \node[state, initial, initial text=] (i) {$i$};
    \node[state, below of=i] (j) {$j$};
    \node[right of=j] (good) {$+1$};
    \node[right of=i] (bad) {$-1$};

    \path 
      (i) edge [loop, out=115, in=65, distance=.8cm] node {$\frac{2}{3}$} (i)
      (i) edge [] node {$\frac{1}{6}$} (bad)
      (i) edge [left] node {$\frac{1}{6}$} (j)
      (j) edge [] node {$\frac{2}{3}$} (good)
      (j) edge [] node {$\frac{1}{3}$} (bad)
    ;
  \end{tikzpicture}
}

\newcommand{\anotherpMCmixedexample}{
  \begin{tikzpicture}[->, node distance=3cm, auto]
    \node[state, initial, initial text=] (i) {$i$};
    \node[state, below of=i, label=below:{$\vdots$}, label=315:{$n \ge 2$}] (n) {$n$};
    \node[state, right of=c, label=below:{$\vdots$}, label=315:{$x \in [0, \frac{1}{n - 1}]$}] (x) {$x$};
    \node[state, right of=i] (b) {$1$};
    \node[right of=x] (yes) {$+1$};
    \node[right of=d] (no) {$-1$};

    \path (i) edge [] node {$p(1-p)^{n-1}$} (n);
    \path (i) edge [bend left] node {$p$} (b);
    \path (b) edge [bend left] node {$1 - q$} (i);
    \path (n) edge [] node {$n - 1$} (x);
    \path (b) edge [] node {$q$} (no);
    \path (x) edge [] node {$xq$} (yes);
    \path (x) edge [right] node {$1-xq$} (b);
  \end{tikzpicture}
}

\pgfdeclareplotmark{hexagon}{%
  \pgfpathmoveto{\pgfqpointpolar{0}{\pgfplotmarksize}}%
  \pgfpathlineto{\pgfqpointpolar{60}{\pgfplotmarksize}}%
  \pgfpathlineto{\pgfqpointpolar{120}{\pgfplotmarksize}}%
  \pgfpathlineto{\pgfqpointpolar{180}{\pgfplotmarksize}}%
  \pgfpathlineto{\pgfqpointpolar{240}{\pgfplotmarksize}}%
  \pgfpathlineto{\pgfqpointpolar{300}{\pgfplotmarksize}}%
  \pgfpathclose
  \pgfusepathqstroke
}

\pgfdeclareplotmark{concentric}{%
  \pgfpathcircle{\pgfpointorigin}{\pgfplotmarksize}%
  \pgfpathcircle{\pgfpointorigin}{0.5\pgfplotmarksize}%
  \pgfusepathqstroke
}

\input{figures.tex}

\usepackage{hyperref}
\usepackage{url}
\usepackage{cite}

\title{\Huge \begin{center} Log-Derivative Trick for \\ Parametric Markov Chains \end{center}}

\author{\large Arthur Gall\textsuperscript{1,2} \\
\And
\large Ichiro Hasuo\textsuperscript{1,2,3} \\
\And
\large Thomas Seiller\textsuperscript{4} \\
}

\newcommand{\ihcomm}[1]{}
\newcommand{\artcomm}[1]{\marginpar{\color{red} TODO}}

\renewenvironment{abstract}{%
  \vskip 0.1in
  \begin{quote}
    \small                      % Slightly smaller font size
    \emph{Abstract.}    % Bigger, bold, and aligned with abstract text
}{%
    \par
  \end{quote}
  \vskip 1ex
}

\begin{document}

\maketitle

\vspace{-.5cm}
\begin{center}
  \small \textsuperscript{1}National Institute of Informatics, Tokyo, Japan \\ 
  \small \textsuperscript{2}SOKENDAI (The Graduate University for Advanced Studies), Kanagawa, Japan \\
  \small \textsuperscript{3}Imiron Co., Ltd., Tokyo, Japan \\
  \small \textsuperscript{4}CNRS, France
\end{center}
\vspace{.1cm}

\begin{abstract}
  Optimizing expected rewards and reachability probabilities in parametric Markov chains (pMCs) is central to the analysis and synthesis of probabilistic systems with loops.
  While traditional \emph{symbolic synthesis} approaches---such as parameter lifting and rational function generation---suffer from scalability bottlenecks, \emph{gradient-based optimization}  emerges as an effective alternative.
  However, existing gradient-based techniques still rely on \emph{symbolic gradient computation}---such as solving large linear systems---over the global state space, inheriting scalability limits. Finite-horizon approximation is possible but introduces uncontrollable bias.

  In this paper, we introduce a scalable Monte Carlo gradient estimator based on a measure-theoretic generalization of the log-derivative trick. 
  We prove that our estimator provides unbiased gradient estimates under almost-sure termination and standard regularity conditions.
  Crucially, our method avoids global state-space materialization and linear system solving entirely: it operates purely along sampled trajectories, requiring only black-box evaluations of local transition probabilities (or densities) and their derivatives.
  Unlike existing pMC techniques that are fundamentally restricted to finite discrete systems, our measure-theoretic framework naturally extends to continuous and general $\sigma$-finite state spaces.
  Empirical evaluations on synthetic discrete pMCs and continuous-state case studies demonstrate that our method scales to state spaces far beyond the reach of existing solvers, significantly reducing execution time without introducing bias.
\end{abstract}

\section{Introduction}
Solving verification and optimization problems over probabilistic programs
is part of a growing research trend~\citep{
%gordon2014probabilistic,
barthe2020foundations,
BaierK08%
% kaminski2018weakest,chakarov2013probabilistic,
% chatterjee2016algorithmic,abate2021learning,bao2022data
}
with a wide range of applications, including
security~\citep{barthe2012probabilistic,
%albarghouthi2018synthesizing,
roy2021learning},
reliability~\citep{carbin2013verifying,wang2021central}
and machine learning~\citep{
%carpenter2017stan,
bingham2019pyro,
vandemeent2018introduction,lew2023adev}.
\emph{Parametric Markov chains (pMCs)} provide a foundational model for probabilistic systems whose transition probabilities depend on a set of global parameters,
capturing parameter dependencies across states.
Since they naturally support loops, they represent a wide subclass of unconditioned probabilistic programs.

\begin{figure}[htbp]
  \centering
  % --- Subfigure 1: Finite MC ---
  \begin{minipage}[b]{0.26\textwidth}
    \centering
    \resizebox{\linewidth}{!}{\MCsimpleexample}
    \caption{A finite MC $T_{1}$}
    \label{fig:mc_simple}
  \end{minipage}\hfill
  % --- Subfigure 2: Finite pMC ---
  \begin{minipage}[b]{0.26\textwidth}
    \centering
    \resizebox{\linewidth}{!}{\pMCsimpleexample}
    \caption{A finite pMC  $T_{2}$}
    \label{fig:pmc_simple}
  \end{minipage}\hfill
  % --- Subfigure 3: Continuous pMC ---
  \begin{minipage}[b]{0.4\textwidth}
    \centering
    \resizebox{\linewidth}{!}{\pMCmixedexample}
    \caption{A continuous pMC  $T_{3}$}
    \label{fig:pmc_mixed}
  \end{minipage}
\end{figure}

Figure~\ref{fig:mc_simple} presents a finite Markov chain (MC); 
Figure~\ref{fig:pmc_simple} presents a finite pMC with a parameter $p$; 
Figure~\ref{fig:pmc_mixed} presents an infinite pMC with a continuum of states $x\in \R_+$, with parameter $p$. 
Here, $-1$ and $+1$ denote \emph{output states} of the system, carrying designated terminal rewards.

Verification and optimization over pMCs have been extensively studied~\citep{daws2004symbolic,
%hahn2011probabilistic,
junges2024parameter%,
%jansen2022parameter
}, with efficient
algorithms and tools~\citep{%quatmann2016parameter,
dehnert2015prophesy,
%cubuktepe2018synthesis,
cubuktepe2022convex,badings2023efficient,
spel2021finding,heck2025generalized%,hensel2022storm
}
applied to reliability analysis~\citep{volk2016advancing,bartocci2011model},
planning under uncertainty~\citep{junges2018finite,polgreen2016data},
and systems biology~\citep{ceska2017precise}.

\paragraph{Optimization over pMCs}
Given a real reward function defined over the output states of a parametric Markov chain ($\{-1,+1\}$ in Figures~\ref{fig:pmc_simple}--\ref{fig:pmc_mixed}), our goal is to synthesize a parameter valuation that maximizes the expected reward.
Existing techniques broadly fall into two classes. 
\begin{itemize}[nosep, leftmargin=1.5em]
  \item \emph{Symbolic synthesis}: Approaches based on parameter lifting, rational function generation, or feasibility region decomposition~\citep{spel2021finding,dehnert2015prophesy,heck2025generalized} offer rigorous global optimality guarantees by symbolically exploring the parameter space. 
However, their computational complexity scales poorly with the number of parameters; state-of-the-art tools rarely scale beyond dozens of parameters~\citep{heck2025generalized}.
  \item \emph{Gradient descent}: When systems involve higher-dimensional parameter spaces, \emph{gradient descent}, while lacking global optimality guarantees, provides a practical and scalable optimization paradigm. 
Here, the central computational challenge lies in efficiently computing the gradient of the expected reward with respect to the parameters.
\end{itemize}

\paragraph{%
%Exact and Approximate Gradients via State-Space Methods
Linear Solver Method for Gradients, Exact and Approximate
}
Recent work~\citep{Heck22} presents what we call the \emph{linear solver method};
% an approach for computing 
it computes exact gradients in finite pMCs by solving a coupled linear system.
Specifically, for an $n$-state pMC with transition matrix $A$ and reward vector $b$, the expected reward vector $x$ satisfies $x = Ax + b$.
Differentiating this relation with respect to a parameter yields $\partial x = (\partial A) x + A \partial x + \partial b$, which gives:
\begin{equation}\label{eq:linearSystemHeck}
  \partial x = (I - A)^{-1} \left( (\partial A) x + \partial b \right)
  \quad\text{where}\quad
  x = (I - A)^{-1} b.
\end{equation}
Various numerical solvers enable the solution of such linear systems, ranging from direct sparse factorization methods~\citep{davis2006direct} to iterative Krylov subspace algorithms~\citep{saad2003iterative}, which are particularly well-suited for large sparse representations. 
These solvers provide exact gradient evaluations (up to numerical precision), yet this exactness incurs a substantial computational cost on large instances.

One can imagine a natural approximation of the linear solver method, namely the \emph{Jacobi method} applied to~\eqref{eq:linearSystemHeck}. It approximates the gradient by unrolling the linear system for a fixed budget of $K$ steps. However, this finite-horizon truncation introduces uncontrollable bias and fails to capture long-horizon behaviors.

% iterative approach to approximate the solution of~\eqref{eq:linearSystemHeck} is the Jacobi method, which computes the gradient by unrolling the linear system for a fixed budget of $K$ steps. 
% However, this finite-horizon truncation introduces uncontrollable bias and fails to capture long-horizon behaviors.

%Crucially, 
The linear solver method, both exact and approximate, 
%\emph{both} the linear solver and Jacobi iterative method 
%share
comes with two fundamental limitations:
\begin{itemize}[nosep, leftmargin=1.5em]
  \item \textbf{Global state-space materialization}: They require explicitly constructing and storing the full transition matrix $A$ (and its derivative 
%matrices
 $\partial A$), which quickly exhausts memory on large models.
  \item \textbf{Restriction to finite discrete systems}: They fundamentally operate on discrete adjacency matrices, rendering them inapplicable to continuous or general state spaces such as Figure~\ref{fig:pmc_mixed}.
\end{itemize}

\paragraph{Our Method: Sampling Gradients via the Log-Derivative Trick}

In this paper, we present a new algorithm that samples the gradient via the \emph{log-derivative trick}~\citep{williams1992simple}, operating entirely along sampled trajectories without state-space materialization and naturally extending to infinite and continuous state spaces.
Note first that the expected reward $\E[r]$ over paths $\pi$ is given by the formula
\[
  \E[r]
  = \sum_{\pi \in \text{Path}} {\color{red} \prob_{\pi}} {\color{blue} r_\pi}
  = \sum_{\pi \in \text{Path}} {\color{red} \prob_{\pi_1} \dots \prob_{\pi_{|\pi|}}} {\color{blue} r_\pi},
\]
where $\prob_{\pi_{i}}$ is the probability of taking the branch $\pi_{i}$, and $\prob_{\pi}$ is that for the path $\pi$. 
We assume that the reward function does not depend on the parameters, i.e.\ $\partial r_\pi = 0$ (for simplicity; lifting this is easy). Then 
differentiation $\partial$ (with respect to parameters) yields
\[
  \partial \E[r]
  = \sum_{\pi \in \text{Path}} \partial \left( \prob_{\pi_1} \dots \prob_{\pi_{|\pi|}} \right) r_\pi
  = \sum_{\pi \in \text{Path}} {\color{red} \prob_{\pi_1} \dots \prob_{\pi_{|\pi|}}} {\color{blue} \partial \log \left( \prob_{\pi_1} \dots \prob_{\pi_{|\pi|}} \right) r_\pi} 
  = \E[\partial \log \prob_{\pi} \cdot r],
\]
where the second equality used the log-derivative trick ($\partial \log f=\partial f/f$). 
The last expression $\E[\partial \log \prob_{\pi} \cdot r]$ implies that the gradient can be estimated by sampling paths $\pi$ directly from the parametric Markov chain
and, by the product rule of logarithm, accumulating the sum of the log-derivatives of the transition probabilities scaled by the reward of the path. Indeed,
\[
  {\color{blue}\partial \log \left( \prob_{\pi_1} \dots \prob_{\pi_{|\pi|}} \right) = \partial \log \prob_{\pi_1} + \dots + \partial \log \prob_{\pi_{|\pi|}}}.
\]
Our algorithm operates as a Monte Carlo gradient estimator~\citep{mohamed2020montecarlogradientestimation},
reducing computational cost at the price of increased variance.
As a theoretical result, we prove that our sampling algorithm provides unbiased gradient estimates under almost-sure termination and standard regularity conditions. 
We also demonstrate experimentally that,
despite the potentially unbounded path lengths arising from cycles,
our algorithm remains scalable and effective in practice for computing an approximation of the gradient. We show that our algorithm is suitable for gradient descent in pMCs and outperforms other methods such as the linear solver method of~\citet{Heck22}. 

Beyond finite systems, our framework naturally generalizes to infinite and continuous state spaces (such as the model in Figure~\ref{fig:pmc_mixed}), which lie entirely outside the scope of methods that require explicit state spaces (such as symbolic synthesis and the linear solver methods).
%state-space methods since they fundamentally rely on finite representations. 
As another theoretical contribution, we generalize the log-derivative trick via the \emph{Radon--Nikodym derivative},  and show that our estimator remains unbiased for any $\sigma$-finite transition structure under almost-sure termination and standard regularity conditions.

Summarizing, the main contribution of this paper is a trajectory-based gradient estimator algorithm for parametric Markov chains. It features the following theoretical and practical characteristics:
\begin{itemize}[nosep, leftmargin=1.5em]
  \item \textbf{Unbiasedness}: 
%    We prove that
 Under almost-sure termination and standard regularity conditions, our estimator yields unbiased gradient estimates. This is unlike the approximate linear solver (see above).

%, entirely avoiding the truncation bias of finite-horizon methods.
  \item \textbf{No state-space materialization needed}: 
    The estimator operates purely along sampled
% execution traces. 
paths. It requires only local evaluations of transition probabilities (or densities) and their derivatives, completely bypassing the need to store entire transition matrices or solve global linear systems.
  \item \textbf{General state spaces}: 
    Formulated in general measure-theoretic terms, our framework
% naturally
 handles discrete, continuous, and hybrid $\sigma$-finite state spaces within a unified mathematical foundation.
  \item \textbf{Empirical scalability and effective synthesis}: 
    We show that our estimator pairs naturally with constrained gradient-descent optimizers. In end-to-end parameter synthesis under fixed time budgets, it quickly boosts rewards by trading exact gradient computations for faster, approximate steps.
It scales to models orders of magnitude larger than linear solvers (and thus than symbolic synthesis). It also optimizes continuous-state pMCs where other methods cannot directly apply.
\end{itemize}

\subsection*{Related Work}
%\paragraph{Gradient estimation for Markov models.}
While the techniques surveyed throughout this paper originate primarily in \emph{formal verification and synthesis}~\citep{BaierK08}, gradient-based optimization on parametric Markov models was already explored in early foundational work from the reinforcement learning and control communities~\citep{MarbachT01,BaxterB01}, relying on the classical \emph{log-derivative trick}~\citep{williams1992simple}. 
However, their algorithmic objectives and structural assumptions diverge fundamentally from ours. 
Specifically, those works address \emph{infinite-horizon, non-terminating} systems, seeking to optimize the \emph{long-run average reward} assuming ergodicity.
To estimate gradients along infinite trajectories, they either assume regenerative renewal cycles (optimizing reward per time step rather than total return) or introduce artificial discount factors that render the resulting gradient estimates inherently biased~\citep{BaxterB01}.
In contrast, the setting studied in this work targets \emph{terminating} parametric Markov chains, where executions almost surely hit an absorbing target state, and aims to optimize terminal rewards or reachability probabilities along these finite trajectories without introducing bias or relying on ergodicity.

\subsection*{Preliminaries and Notations} \label{subsec:preliminaries}

We adopt the following measure-theoretic, algebraic, and differential conventions.

\begin{itemize}[nosep, leftmargin=1.5em]
  \item \textbf{Sequences and Paths:}
    For a set $S$, $S^n$ denotes the $n$-fold Cartesian product (with $S^0 = \{ \epsilon \}$ containing the empty sequence $\epsilon$), and $S^* = \coprod_{n = 0}^\infty S^n$ denotes the set of all finite sequences (paths) over $S$. For a sequence $s = (s_1, \dots, s_n) \in S^n$, its length is denoted by $|s| = n$.
  
  \item \textbf{Measurable Spaces and Coproducts:}
    A measurable space is denoted by $(X, \F_X)$. Given two measurable spaces $(X, \F_X)$ and $(Y, \F_Y)$, their coproduct (disjoint union) is denoted by $(X + Y, \F_{X+Y})$, where $X+Y$ is the disjoint union and  $\F_{X+Y} = \{ A_X + A_Y \mid A_X \in \F_X, A_Y \in \F_Y \}$.
    For a measurable space $(S, \F_S)$, the sequence space $S^* = \coprod_{n=0}^\infty S^n$ is equipped with the coproduct $\sigma$-algebra $\F_{S^*}$, where $A \in \F_{S^*}$ if and only if $A \cap S^n \in \F_{S^n}$ for all $n \ge 0$.

  \item \textbf{Measures and Dirac Notation:}
    A finite signed measure $\mu$ on $(X, \F_X)$ is a $\sigma$-additive function $\mu : \F_X \to \R$ with $\mu(\emptyset) = 0$.
    Its total variation measure is denoted by $|\mu|$.
%    For discrete sets,
We let $\bra{x}$ denote the Dirac delta measure concentrated at state $x$, defined by $\bra{x}(A) = 1$ if $x \in A$ and $0$ otherwise. A discrete distribution is written as a formal sum $\sum_{x} c_x \bra{x}$.

  \item \textbf{Differentiation and Integration:}
    For functions or parameterized families over an open domain $P \subseteq \R^k$, $\partial_v$ denotes the directional derivative along a direction vector $v \in \R^k$, $\partial_i$ denotes the partial derivative with respect to the $i$-th coordinate (i.e., along the standard basis vector $e_i \in \R^k$), 
    and $\partial_p$ denotes the scalar derivative with respect to a single parameter $p$ (when $k = 1$). 
    We write $\dx$ (or $\ds = (\ds_1, \dots, \ds_n)$) for infinitesimal volume elements when integrating over a state space $X$ (or product space $S^n$).
\end{itemize}

\section{Optimization of Parametric Signed Markov Chains} \label{sec:Signed Markov chains}

\subsection{Parametric Markov Kernels and Signed Kernels} \label{subsec:Parametric signed kernels}
We develop our theory in general measure theoretic terms; it accommodates continuous state spaces.

To formalize parametric Markov chains over arbitrary state spaces (including continuous ones), we rely on the framework of Markov kernels. 
%and their signed extensions.
Let $(X, \F_X)$ and $(Y, \F_Y)$ be two measurable spaces, and let $P \subseteq \R^k$ denote an open parameter space.
A \emph{parametric Markov kernel} $\kappa : X \rightsquigarrow Y$ is a function $\kappa : P \times X \times \F_Y \to [0, 1]$ such that for every parameter $p \in P$ and state $x \in X$, $\kappa(p, x, -)$ is a probability measure on $(Y, \F_Y)$, and for every $p \in P$ and measurable set $A \in \F_Y$, $\kappa(p, -, A)$ is $\F_X$-measurable.
Given two parametric kernels $\kappa : X \rightsquigarrow Y$ and $\kappa' : Y \rightsquigarrow Z$ over the same parameter space $P$, their composition $\kappa \kappa' : X \rightsquigarrow Z$ is defined by
\[
  \kappa \kappa'(p, x, \dz) := \int_{y \in Y} \kappa(p, x, \dy) \kappa'(p, y, \dz).
\]
In the discrete and finite case, this formulation reduces to standard stochastic transition matrices, with kernel composition corresponding to matrix multiplication.
For instance, a Bernoulli distribution $\Bern(p) = (1 - p) \bra{0} + p \bra{1}$ parameterized by $p \in [0, 1]$ can be represented as the kernel $\kappa(p, *, -) = (1 - p) \bra{0} + p \bra{1}$, where $*$ denotes a single input state.

\paragraph{Signed Extension}
To compute and propagate gradients through such systems, we must differentiate kernels with respect to their parameters. Differentiating a probability measure yields quantities with both positive and negative values (for instance, $\partial_p \Bern(p) = -\bra{0} + \bra{1}$). 
We therefore generalize to \emph{parametric signed kernels} $\kappa : P \times X \times \F_Y \to \R$, where $\kappa(p, x, -)$ is a finite-valued signed measure for all $p \in P$ and $x \in X$.
Kernel composition extends naturally to signed kernels whenever the integral converges.
Under this generalized framework, the derivative of the Bernoulli distribution is represented directly by the signed kernel $\partial_p \kappa(p, *, -) = -\bra{0} + \bra{1}$.
Finally, we say that a parametric kernel $\kappa$ is \emph{differentiable} if, for every $x \in X$ and $A \in \F_Y$, the scalar map $p \mapsto \kappa(p, x, A)$ is differentiable on $P$, and all its directional derivatives define signed kernels.

\subsection{Absolute Differentiability} \label{subsec:Absolute differentiability}
In this section, we introduce a mild regularity condition for parametric kernels, termed \emph{absolute differentiability}.
This condition relies on strong standard measure-theoretic tools,
namely absolute continuity and the Radon--Nikodym derivative,
which arise naturally when differentiating parametric kernels.
Moreover, this property is satisfied by most standard continuous and discrete distributions (e.g., Gaussian, exponential, Bernoulli, geometric, etc.).

\begin{definition}[Absolute Continuity]
  Let $\mu$ and $\nu$ be two signed measures over the same measurable space $(X, \F_X)$.
The signed measure  $\nu$ is said to be \emph{absolutely continuous} with respect to $\mu$, written $\nu \ll \mu$,
  if $|\mu|(A) = 0$ implies $|\nu|(A) = 0$ for all $A \in \F_X$.
\end{definition}

\begin{theorem}[Radon--Nikodym Theorem]
  Let $\mu$ and $\nu$ be two $\sigma$-finite signed measures over the same measurable space $(X, \F_X)$.
  Then $\nu$ is absolutely continuous with respect to $\mu$ if and only if there exists a measurable function $f : X \to \R$ such that $\d\nu = f \d\mu$.
  Such a function $f$ is unique up to $\mu$-null sets and is called the \emph{Radon--Nikodym derivative}
  of $\nu$ with respect to $\mu$, denoted by $\tfrac{\d\nu}{\d\mu}$.
\end{theorem}

\begin{definition}[Absolute Differentiability]
  Let $\kappa : X \rightsquigarrow Y$ be a differentiable parametric signed kernel with parameter space $P \subseteq \R^k$.
The kernel  $\kappa$ is said to be \emph{absolutely differentiable} if, for all directional vectors $v \in \R^k$, parameters $p \in P$ and states $x \in X$, 
  \[
    \partial_v \kappa(p, x, -) \ll \kappa(p, x, -) 
  \]
\end{definition}

\begin{remark}
  Intuitively, the absolute differentiability condition states that the derivative 
  must vanish on any set where the kernel itself assigns zero probability 
  (i.e., the derivative cannot be non-zero if the base probability is zero).
  Most standard discrete and continuous parametric kernels satisfy this property.
  Their associated Radon--Nikodym derivative corresponds to the derivative of the logarithm 
  of the transition probabilities (in the discrete case) or transition densities (in the continuous case).
\end{remark}

\begin{example}
  Let $\kappa : \{*\} \rightsquigarrow \{0, 1\}$ be the kernel parameterized by $p \in (0, 1)$ defined as 
  $\kappa(p, *, -) = (1 - \sqrt{p})\bra{0} + \sqrt{p}\bra{1}$.
  Its derivative is the signed kernel $\partial_p\kappa : \{*\} \rightsquigarrow \{0, 1\}$ given by 
  $\partial_p\kappa(p, *, -) = -\tfrac{1}{2\sqrt{p}} \bra{0} + \tfrac{1}{2\sqrt{p}} \bra{1}$.
  The Radon--Nikodym derivative is the measurable function
  \[
    \frac{\d\partial_p\kappa(p, *, -)}{\d\kappa(p, *, -)}
    :\; 0 \;\longmapsto\; \frac{\partial_p\kappa(p, *, \{0\})}{\kappa(p, *, \{0\})} 
    = -\frac{p + \sqrt{p}}{2p(1-p)}
    \quad;\qquad 1 \;\longmapsto\; \frac{\partial_p\kappa(p, *, \{1\})}{\kappa(p, *, \{1\})} 
    = \frac{1}{2p}.
  \]
\end{example}

\begin{example}
  Let $\kappa : \{*\} \rightsquigarrow \R_+$ define an Exponential distribution $\kappa(p, *, -) = \Exp(p)$ for $p > 0$
  with probability density function $f_p : x \mapsto pe^{-px}$.
  The derivative of this density with respect to $p$ is
  $\partial_p f_p : x \mapsto (1 - px)e^{-px}$.
  The Radon--Nikodym derivative is the measurable function
  \[
    \frac{\d\partial_p\kappa(p, *, -)}{\d\kappa(p, *, -)}
    :\; x \;\longmapsto\; \frac{\partial_p f_p(x)}{f_p(x)}
    = \frac{1 - px}{p}
    = \frac{1}{p} - x.
  \]
\end{example}

\subsection{Parametric (Signed) Markov Chains} \label{subsec:Signed Markov chains}

%In this section, 
We formally introduce
% the formal definition of
 \emph{parametric  Markov chains} (pMCs), as a special case of the \emph{signed} pMCs.
\begin{remark}[Scope and Comparison with Literature]
  In the formal verification literature, parametric Markov chains are commonly defined as finite-state, discrete-time systems whose transition probabilities are rational or polynomial functions of parameters.
  In this work, our definition is more general: we consider \emph{discrete-time parametric Markov chains over $\sigma$-finite state spaces} equipped with explicit input and output state boundaries.
\end{remark}

\begin{definition}[Parametric (Signed) Markov Chain]\label{def:pSMC}
  Let $(X, \F_X)$, $(S, \F_S)$, and $(Y, \F_Y)$ be measurable spaces representing 
  the \emph{input states}, \emph{internal states}, and \emph{output (halting) states}, respectively.
  A \emph{parametric signed Markov chain} is a parametric signed kernel of the form
  \[
    T : X + S \rightsquigarrow S + Y
  \]
  parameterized by $p \in P \subseteq \R^k$, where $+$ denotes the disjoint union of measurable spaces.
  The transition kernel $T$ canonically induces a \emph{path kernel} $\mupath_T : X \rightsquigarrow S^* \times Y$
  defined for any parameter $p \in P$, initial state $x \in X$, internal path $s \in S^n$ and output state $y \in Y$ as
  \[
    \mupath_T(p, x, \ds, \dy) := 
    \begin{cases}
      T(p, x, \dy)
      &\text{if } n = 0, \\
      T(p, x, \ds_1) \displaystyle\prod_{i = 1}^{n - 1} T(p, s_i, \ds_{i + 1}) \, T(p, s_n, \dy)
      &\text{if } n > 0.
    \end{cases}
  \]
Moreover, it induces an \emph{output (execution) kernel} $\muout_T : X \rightsquigarrow Y$, obtained by marginalizing over all finite internal paths
  \[
    \muout_T(p, x, \dy)
    \;:=\; \mupath_T(p, x, S^*, \dy)
    \;=\; \sum_{n = 0}^\infty \int_{s \in S^n} \mupath_T(p, x, \ds, \dy).
  \]

%We define
 \emph{Parametric Markov chains} (pMC) are a special case of parametric signed Markov chains, such that the associated kernel is a parametric Markov kernel (in the non-extended sense; see Section~\ref{subsec:Parametric signed kernels}).
\end{definition}

\begin{example}
  Consider the finite parametric Markov chain shown in Figure~\ref{fig:pmc_simple}, 
  with input space $X = \{*\}$, internal states $S = \{i, j\}$, and output states $Y = \{-1, +1\}$.
  The transition kernel $T : X + S \rightsquigarrow S + Y$ is represented by the matrix
\[
  T(p, -, -) \;=\;
  \begin{pmatrix}
    0 & 1 & 0 & 0 \\
    1 - p & \frac{p}{2} & \frac{p}{2} & 0 \\
    0 & 0 & p & 1 - p
  \end{pmatrix}.
\]

  Evaluating the sum of paths yields the output kernel $\muout_T : \{*\} \rightsquigarrow \{-1, +1\}$ given by
  \[
    \muout_T(p, *, -) \;=\; \frac{1 + p}{2} \bra{-1} + \frac{1 - p}{2} \bra{+1}.
  \]
\end{example}

\subsection{The Log-Derivative Trick} \label{subsec:Log-derivative trick}

Having formalized parametric Markov chains and their differentiation via signed kernels,
we now address the computation of the directional derivative $\partial_v \muout_T(p, x, -)$, 
which enables gradient-based optimization of expected rewards.
In this section, we establish a measure-theoretic generalization of the log-derivative trick~\citep{williams1992simple}, 
expressing the Radon--Nikodym derivative of the path kernel $\smash{\mupath_T}$ as an additive accumulation along paths.

\begin{theorem}[Log-Derivative Trick] \label{thm:log_derivative_trick}
  Let $T : X + S \rightsquigarrow S + Y$ be an absolutely differentiable parametric signed Markov chain over an open parameter space $P \subseteq \R^k$.
  The path kernel $\smash{\mupath_T}$ is absolutely differentiable. Moreover, for all directional vectors $v \in \R^k$, parameters $p \in P$, initial states $x \in X$, finite internal paths $s \in S^n$, output states $y \in Y$, and $n > 0$,
  \[
    \frac{\d \partial_v \mupath_T(p, x, -, -)}{\d \mupath_T(p, x, -, -)}(s, y)
    \!=\! \frac{\d \partial_v T(p, x, -)}{\d T(p, x, -)}(s_1)
    + \sum_{i = 1}^{n - 1} \frac{\d \partial_v T(p, s_i, -)}{\d T(p, s_i, -)}(s_{i + 1})
    + \frac{\d \partial_v T(p, s_n, -)}{\d T(p, s_n, -)}(y).
  \]
  For $n = 0$ (direct transition to an output state without visiting $S$), the expression reduces to
  \[
    \frac{\d \partial_v \mupath_T(p, x, -, -)}{\d \mupath_T(p, x, -, -)}(\epsilon, y)
    = \frac{\d \partial_v T(p, x, -)}{\d T(p, x, -)}(y).
  \] 
\end{theorem}

\begin{proof}
  Let $n \ge 1$. For clarity, we introduce the shorthand notations
  \begin{align*}
    V_0,\, V_1,\, \dots,\, V_n &:= T(p, x, \ds_1),\, T(p, s_1, \ds_2),\, \dots,\, T(p, s_n, \dy) \\
    D_0,\, D_1,\, \dots,\, D_n &:= \frac{\d \partial_v T(p, x, -)}{\d T(p, x, -)}(s_1),\, \frac{\d \partial_v T(p, s_1, -)}{\d T(p, s_1, -)}(s_2),\, \dots,\, \frac{\d \partial_v T(p, s_n, -)}{\d T(p, s_n, -)}(y)
  \end{align*}
  By the definition of the Radon--Nikodym derivative, we have $\partial_v V_i = D_i V_i$ for all $i \in \{0, \dots, n\}$.
  Applying the Leibniz product rule, we obtain
  \begin{align*}
    \partial_v \mupath_T(p, x, \ds, \dy)
    &= \partial_v \left( V_0 \cdots V_n \right) \\
    &= \partial_v V_0 V_1 \cdots V_n + V_0 \partial_v V_1 V_2 \cdots V_n + \dots + V_0 \cdots V_{n - 1} \partial_v V_n \\
    &= D_0 V_0 V_1 \cdots V_n + V_0 D_1 V_1 V_2 \cdots V_n + \dots + V_0 \cdots V_{n - 1} D_n V_n \\
    &= (D_0 + \dots + D_n) \cdot V_0 \cdots V_n \\
    &= (D_0 + \dots + D_n) \cdot \mupath_T(p, x, \ds, \dy).
  \end{align*}
  Since the sum $D_0 + \dots + D_n$ is measurable, it uniquely defines the Radon--Nikodym derivative
  \[
    \frac{\d \partial_v \mupath_T(p, x, -, -)}{\d \mupath_T(p, x, -, -)}(s, y) = D_0 + \dots + D_n.
  \]
  The base case $n = 0$ follows immediately.
\end{proof}

\begin{corollary} \label{cor:derivative of muout}
  Let $T$ be an absolutely differentiable parametric signed Markov chain over $P \subseteq \R^k$.
  Under standard regularity conditions allowing the interchange of differentiation and summation (satisfied when paths terminate almost surely with a bounded expected length), we have
  for all $v \in \R^k$, $p \in P$, $x \in X$ and $y \in Y$,
\begin{align*}
    \partial_v \muout_T(p, x, \dy)
    = \int_{s \in S^*} \frac{\d \partial_v \mupath_T(p, x, -, -)}{\d \mupath_T(p, x, -, -)}(s, y) \cdot \mupath_T(p, x, \ds, \dy).
\end{align*}
\end{corollary}
\begin{proof}
  Differentiating $\muout_T$ with respect to $p$ along direction $v$ and exchanging differentiation and integration via the dominated convergence theorem yields
  \begin{align*}
    \partial_v \muout_T(p, x, \dy)
    &= \partial_v \int_{s \in S^*} \mupath_T(p, x, \ds, \dy) \\
    &= \int_{s \in S^*} \partial_v \mupath_T(p, x, \ds, \dy) \\
    &= \int_{s \in S^*} \frac{\d \partial_v \mupath_T(p, x, -, -)}{\d \mupath_T(p, x, -, -)}(s, y) \cdot \mupath_T(p, x, \ds, \dy).
    \qedhere
  \end{align*}
\end{proof}

\subsection{Gradient Sampling and Optimization}  \label{subsec:Optimization of Markov chains}

%With the proper definitions above, we can just say let T be a pMC (i.e. probabilistic parametric signed Markov chain)
Throughout this section, let $T : X + S \rightsquigarrow S + Y$ be 
an absolutely differentiable
% \emph{probabilistic} pMC 
parametric Markov chain---in the restricted sense at the end of Definition~\ref{def:pSMC}---over an open parameter space $P \subseteq \R^k$. Here, each $T(p, a, -)$ is a probability measure that can be sampled from, 
while its partial derivatives $\partial_i T$ form parametric \emph{signed} Markov chains.
Leveraging the log-derivative identity from Corollary~\ref{cor:derivative of muout}, 
we construct a Monte Carlo algorithm to estimate the gradient measure
\[
  \nabla_p \muout_T(p, x, -) := \left( \partial_1 \muout_T(p, x, -),\, \dots,\, \partial_k \muout_T(p, x, -) \right)
\]
at any initial state $x \in X$ and parameter $p \in P \subseteq \R^k$ without solving linear systems or constructing the full state-space transition matrix.

\begin{definition}[Transition Gradient Vector] \label{def:score_vector}
  Let $T : X + S \rightsquigarrow S + Y$ be an absolutely differentiable parametric Markov chain over an open parameter space $P \subseteq \R^k$.
  The \emph{transition gradient vector} $w_T : P \times (X + S) \times (S + Y) \to \R^k$ is defined by
  \[
    w_T(p, a, b)  \;\coloneqq\;
    \left( 
      \frac{\d \partial_1 T(p, a, -)}{\d T(p, a, -)}(b),\,
      \dots,\,
      \frac{\d \partial_k T(p, a, -)}{\d T(p, a, -)}(b)
    \right).
  \]
\end{definition}

\begin{algorithm}[tbp]
  \caption{Log-Derivative Gradient Sampler for pMCs}
  \label{alg:log_derivative_sampler}
  \begin{algorithmic}[1]
    \Require  pMC $T : X + S \rightsquigarrow S + Y$, parameter $p \in P\subseteq \R^{k}$, initial state $x \in X$
    \Ensure Single weighted sample $(w_{\mathrm{total}}, y)$, with $w_{\mathrm{total}}\in \R^{k}$ and $y\in Y$,  whose empirical measure
 $w_{\mathrm{total}} 
%\delta_y
\bra{y}$ 
satisfies, for each $A \in \F_Y$,
  {$\E\big[w_{\mathrm{total}} \cdot \mathbf{1}_{A}(y)\big] = \nabla_p \muout_T(p, x, A)$} 
    \State $s \gets x$
    \State $w_{\mathrm{total}} \gets \mathbf{0}_{\R^k}$
    \While{$s \notin Y$}
      \State Sample next state $s' \sim T(p, s, -)$
      \State $w_{\mathrm{total}} \gets w_{\mathrm{total}} + w_T(p, s, s')$ \Comment{$w_T$ is from Definition~\ref{def:score_vector}}
      \State $s \gets s'$
    \EndWhile
    \State $y\gets s$   \Comment{$y\in Y$ is the output state}
    \State \Return $(w_{\mathrm{total}}, y)$ 
%\Comment{$s \in Y$ is a terminal output state}
  \end{algorithmic}
\end{algorithm}

Our algorithm is in Algorithm~\ref{alg:log_derivative_sampler}.
At each step, 
%Algorithm~\ref{alg:log_derivative_sampler} 
it incrementally accumulates in $w_{\mathrm{total}}$ the local transition gradient vectors $w_T(p, s, s')$, 
reproducing the path sum in Theorem~\ref{thm:log_derivative_trick}. 
The returned pair $(w_{\mathrm{total}}, y)$ defines a random signed measure $w_{\mathrm{total}} \bra{y}$; taking its expectation over multiple sampled trajectories yields an unbiased estimate of the true gradient measure $\nabla_p \muout_T(p, x, -)$, with correctness and unbiasedness following directly from Corollary~\ref{cor:derivative of muout}.

In particular, given a terminal reward function $R : Y \to \R$, 
an unbiased estimate of the expected reward gradient $\nabla_p \E_{y \sim \muout_T(p, x, -)}[R(y)]$ is directly obtained as $w_{\mathrm{total}} \cdot R(y)$. 
While this framework extends directly to cumulative and parameter-dependent transition rewards, 
we evaluate our algorithm under this terminal reward setting in the subsequent sections, 
presenting our experimental methodology and empirical performance against other approaches.

\section{Empirical Evaluation}
%(** Please remind the overall picture. Synthesis: symbolic or gradient; gradient computation: symbolic or sampling **)
Given a pMC and a reward function defined over its output states, 
our goal is to synthesize a parameter valuation that maximizes the expected reward.

The two main approaches to tackle this optimization problem are
\emph{symbolic synthesis} such as parameter lifting, rational function generation and feasibility region decomposition~\citep{spel2021finding,dehnert2015prophesy,heck2025generalized},
with global optimality guarantees but that do not scale beyond dozens of parameters,
and \emph{gradient descent} (based on gradient computation algorithms),
with no global optimality guarantees but that scales to higher-dimensional parameter space.

Gradient computation algorithms themselves are subdivided into
exact methods such as the linear solver method~\citep{Heck22}
and approximate methods such as our Monte Carlo approach (Algorithm~\ref{alg:log_derivative_sampler}) and the Jacobi method.

Since our synthetic benchmark involved high-dimensional parameter spaces, 
to evaluate the empirical efficiency and precision of our proposed sampling-based framework, 
%we compare it with several gradient computation techniques for parametric Markov chains (pMCs).
we compare it with other gradient computation techniques 
for gradient computation and gradient descent parameter synthesis.
\begin{itemize}[nosep, leftmargin=1.5em]
  \item \textbf{\SparseLU{} and \BiCGSTAB{}:} Exact \emph{linear solver methods} for gradient computation, following~\citet{Heck22}. For solving the linear system~(\ref{eq:linearSystemHeck}), 
    we evaluate both a direct sparse LU factorization (\SparseLU{}) and an iterative Krylov subspace method (\BiCGSTAB{}).
  \item \textbf{\Jacobi{}:} 
    %A finite-horizon iterative approximation for the linear system~(\ref{eq:linearSystemHeck}); it truncates the evaluation after a fixed number $K$ of matrix-vector multiplication steps.
    Approximates the gradient by unrolling the linear system~(\ref{eq:linearSystemHeck}) for a fixed budget of $K$ matrix-vector multiplication steps,
    effectively truncating the computation to a finite horizon.
  \item \textbf{\LogDerivative{} (Ours):} 
    Our proposed sampling-based estimator, which computes an unbiased gradient estimate by simulating a budget of  N  paths using the path-space log-derivative trick.
  \item \textbf{\UniformPath{}:} 
    An importance-sampling baseline where successor states are drawn uniformly at random regardless of transition probabilities,
    parameterized by the number of sampled paths $N$. 
    Specifically, given $d$ outgoing transitions with probabilities $p_1, \dots, p_d$, 
    branch $i$ is traversed with uniform probability $\tfrac{1}{d}$ and reweighted by $d \cdot p_i$ to maintain unbiasedness.
\end{itemize}

\paragraph{Benchmarks}
We evaluate our approach on both synthetic discrete pMCs and continuous-state benchmarks.
For the discrete setting, we generate synthetic pMCs using a procedural generator in Python. 
Each model fixes an underlying sparse graph topology and assigns to each transition a degree-$1$ polynomial in a single parameter $X_i \in \{X_1, \dots, X_k\}$. 
The transitions are constructed such that outgoing probabilities algebraically sum to $1$ and evaluate to values in $(0,1)$ around an initial parameter value $p_0$. 
Sequential and cyclic paths naturally compose these local degree-$1$ transitions into non-trivial multivariate rational functions. 
The generated models are exported as \textsc{Prism} files; full details of the generation pipeline are provided in Appendix~\ref{appendix:benchmarkGeneration}. 
Continuous-state case studies (such as the model in Figure~\ref{fig:pmc_mixed}) are described in Figure~\ref{fig:RQ5}.

\paragraph{Implementation}%\label{subsec:impl}
All code, benchmarks, and experimental scripts are open-source and publicly available at \url{https://github.com/Gallinou/pMC-optimization-artifact}.
All  gradient computation methods along with the optimizers are implemented in C++ using sparse data structures. 
To ensure a fair runtime comparison under identical data layouts, we implemented the exact baseline solvers (\SparseLU{}, \BiCGSTAB{}) directly within this unified C++ framework.

Symbolic differentiation and evaluation of individual parametric transition expressions is handled using \texttt{SymEngine}\footnote{\url{https://github.com/symengine/symengine}}, 
whereas sparse linear systems are solved using \texttt{Eigen::Sparse}\footnote{\url{https://eigen.tuxfamily.org}} (utilizing both \texttt{SparseLU} and \texttt{BiCGSTAB}). 
All experiments were executed on an AWS EC2 \texttt{m6i.large} instance (x86\_64 architecture, 2 vCPUs powered by an Intel Xeon Platinum 8375C @ 3.5\,GHz, 8\,GiB RAM) constrained to a single core to ensure strictly deterministic, single-threaded execution.

\paragraph{RQ1: What is the Speed and Precision of Our Algorithm in Gradient Estimation?}
The primary motivation for our Monte Carlo sampler is to trade absolute precision for computational efficiency. 
We therefore evaluate: (1) whether it is faster than exact linear solver methods (\SparseLU{}, \BiCGSTAB{}), and (2) whether it remains sufficiently accurate, yielding gradient estimates close to the exact solutions.
Moreover, because computing the gradient scales poorly with the parameter dimension $k$, we assess runtime behavior as $k$ increases.

Figure~\ref{fig:RQ1} summarizes these results. Here, execution time denotes the duration required to compute the gradient at the initial parameter value $p_0$. 
As shown in the left plot, runtime increases for all methods as the number of parameters grows. 
The center plot displays runtime normalized against \SparseLU{} as a baseline, showing a clear performance advantage for the sampling-based methods (\LogDerivative{} and \UniformPath{}) over the linear solver methods (\SparseLU{}, \BiCGSTAB{}, \Jacobi{}), an advantage that widens considerably for larger parameter space dimensions.

The right-hand table reports the accuracy of each method. 
Because \SparseLU{} produces the exact gradient, we use its output as the ground truth $g$ and compare it against the estimated gradient $g'$ produced by each method using two different metrics:
\begin{itemize}[nosep, leftmargin=1.5em]
  \item \textbf{The Normalized Distance ($d$)} measures the relative error in both magnitude and direction:
    \[
      d(g, g') := \frac{\|g - g'\|_2}{\|g\|_2}.
    \]
    It penalizes scale mismatch and angular divergence (lower is better).

  \item \textbf{The Cosine Similarity ($\cos$)} evaluates the directional alignment irrespective of scale:
    \[
      \cos(g, g') := \frac{\langle g, g' \rangle}{\|g\|_2 \, \|g'\|_2} \in [-1, 1].
    \]
    Here, $1$ denotes identical direction (higher is better).
\end{itemize}
Between the two sampling-based methods, \LogDerivative{} consistently outperforms the precision of \UniformPath{}. 
In particular, with $10^5$ sample paths, \LogDerivative{} achieves a normalized distance of $0.050$ and a cosine similarity of $0.998$, demonstrating near-perfect directional alignment and high precision relative to the exact gradient.
\RQone

\paragraph{RQ2: What is the Overall Performance of Our Algorithm in Parameter Synthesis?}
Since our benchmark generates pMCs with high-dimensional parameter spaces,
for which symbolic methods do not scale well,
we compare the gradient computation methods---\SparseLU{}, \BiCGSTAB{}, \LogDerivative{}, \UniformPath{} and \Jacobi{}---
at gradient-based parameter synthesis.
Explicitly, we compare the reward improvement over execution time of full optimization runs across 100 benchmark pMCs with 5,000 states and 100 parameters under a strict 100-second wall-clock timeout.

Since the result is averaged over 100 pMCs, we introduce the notion of \emph{normalized relative reward} $\rho$.
%\item \textbf{Normalized Reward Improvement ($\rho$):} 
%        Quantifies the relative effectiveness of multi-step gradient optimization. 
Given an initial expected reward $r_0$, final expected rewards $\{r_m\}_{m \in M}$ achieved across all evaluated methods $m \in M$,
and the evaluated expected reward $r$, $\rho(r)$ is defined as
\[
  \rho(r) := \frac{r - r_0}{\max_{m \in M} (r_{m} - r_0)}
\]
%        Under this normalization, 
%        the best-performing optimizer on any given instance attains $\rho = 1.0$, 
%        while a method achieving half of that gain achieves $\rho = 0.5$, 
%        providing a scale-free metric that can be averaged across heterogeneous models.

Here, execution time accounts strictly for gradient estimation and optimizer steps. 
It excludes preprocessing (such as parsing \textsc{Prism} files and building internal sparse representations) and postprocessing (computing the exact reachability reward after each update via $x = (I - A)^{-1}b$ from \eqref{eq:linearSystemHeck}).

Because parameter values must remain within valid probability intervals (such as $[0, 1]$), standard unconstrained optimizers fail: as shown in Figure~\ref{fig:RQ2plus}, \texttt{AdaGrad}, \texttt{Adam}, and \texttt{AdamW} quickly drive parameters outside valid bounds and terminate within seconds. 
Standard \texttt{SGD} and step-clamped SGD (\texttt{SmallSGD}) remain inside the valid domain but make very slow progress. 
In contrast, constrained variants—\texttt{C-AdaGrad} and \texttt{C-AdamW} (with default learning rates $10^{-2}$ and $10^{-3}$)—consistently achieve fast reward improvement for both exact (\SparseLU{}) and sampling-based (\LogDerivative{}) gradients,
despite projecting values and no longer following the exact gradient direction.
We therefore use \texttt{C-AdaGrad} and \texttt{C-AdamW} as the underlying optimizers for our comparison.
\RQtwoplus

Figure~\ref{fig:RQ2} depicts the cumulative normalized relative reward improvement over time across the different exact and approximate gradient estimators. 
\LogDerivative{} substantially outperforms exact linear solver methods (\SparseLU{}, \BiCGSTAB{}) and the sampling baseline \UniformPath{}. 
While exact solvers compute true gradients, each step is slow: within 100 seconds, they execute very few steps and achieve minimal cumulative reward gain. 
Conversely, while \UniformPath{} executes quickly, its precision ceiling (observed in RQ3) limits its progress.

Crucially, low-sample configurations of our estimator (\LogDerivative{} with $N \in \{10, 10^2\}$) achieve the fastest initial reward gains. 
This illustrates a classic trade-off: ``the drunk guy will get home before the wise, though he might get stuck in front of his door.'' 
%(Arthur) We can remove it if you want, but I think its a good mental image to have in mind
With small $N$, the noisy updates are cheap enough to take hundreds of steps and quickly get close to the best reward, even if the noise prevents them from reaching the exact optimum. 
Larger sample sizes take fewer steps, but their lower noise allows them to get closer to the optimum.
\RQtwo

\paragraph{RQ3: How Does the Number of Sampled Paths $N$ Affect the Performance of Gradient Estimation?}
The two sampling-based methods (our proposed \LogDerivative{} and the baseline \UniformPath{}) rely on the number of sampled paths $N$ as their primary hyperparameter. 
We evaluate the sensitivity of gradient estimation runtime and precision for different values of $N$.

As shown in Figure~\ref{fig:RQ3}, the execution time of both sampling methods exhibits strongly sublinear growth with respect to the sample size $N$. Because transition derivatives are computed once and cached, increasing $N$ by five orders of magnitude (from $10$ to $10^6$) results in less than a two-orders-of-magnitude increase in runtime across all benchmark sizes.

Regarding estimation quality evaluated against the exact \SparseLU{} solution via normalized distance and cosine similarity, \UniformPath{} can be slightly more precise for very small $N$, presumably because uniform sampling may hit rare yet high-impact paths early on. However, despite its unbiasedness, its precision quickly plateaus whereas \LogDerivative{} exhibits no such ceiling and steadily converges to the exact gradient as $N$ increases.
\RQthree

\paragraph{RQ4: How Does the Number of States Affect Performance?}
Figure~\ref{fig:RQ4} illustrates scalability as the state space grows from $1$ to over $2.6 \times 10^5$ states. 
While exact numerical solvers such as \SparseLU{} suffer sharply from state-space explosion—exhibiting steep, polynomial runtime growth exceeding $200$ seconds at $2.6 \times 10^5$ states—\LogDerivative{} scales far more favorably, since its computational cost depends only on the average time required to simulate a path rather than on the total number of states or transitions.
Furthermore, as shown in the accompanying table, \LogDerivative{} maintains high estimation quality across these benchmarks, achieving a cosine similarity of $1.000$ (normalized distance $0.013$) at $N=10^6$.
\RQfour

\paragraph{RQ5: Can Our Algorithm Handle pMCs with Continuous State Spaces?}
Our theory, as well as our Algorithm~\ref{alg:log_derivative_sampler}, supports continuous pMCs. Here we experimentally confirm the algorithm's capability.
We evaluate our algorithm on two continuous pMC benchmarks: $T_{3}$ (the mixed discrete--continuous pMC from Figure~\ref{fig:pmc_mixed}) with a single parameter $p$, and a two-parameter continuous pMC $T_{4}$ with parameters $(p, q)$. To monitor optimization progress and verify optimality, we derived closed-form analytical expressions for the exact reachability rewards in both models.

The graphs in Figure~\ref{fig:RQ5} show that, indeed, our algorithm can optimize in continuous pMCs effectively and moreover converges to the theoretical maximum reward, respectively $2\smash{\sqrt{2}}-3$ at $p = \tfrac{1}{2}$ and $\smash{\tfrac{1}{e} - 1}$ at $(p, q) = \smash{\left(\tfrac{1}{e}, 1\right)}$, using a constrained-AdamW optimizer.
Note that existing methods such as symbolic synthesis~\citep{spel2021finding,dehnert2015prophesy,heck2025generalized} and exact linear solver methods~\citep{Heck22} require finite matrix representations and are not directly applicable to infinite or continuous state spaces.
\RQfive

\section{Conclusions and Future Work}

To the best of our knowledge, this establishes the first fully trajectory-sampling framework specifically tailored for verifying and optimizing terminating pMCs over general measurable state spaces.
By leveraging the log-derivative trick and generalizing it via Radon--Nikodym derivatives, our approach naturally accommodates both discrete and continuous state spaces with formal unbiasedness guarantees. 
Empirically, our results distinguish two key benefits: first, for gradient computation, our estimator yields accurate gradient approximations orders of magnitude faster than exact linear solvers; second, for end-to-end optimization, this speed--precision trade-off accelerates gradient descent, enabling rapid parameter synthesis on models far beyond the reach of exact techniques.

\paragraph{Future work}
%Future work includes extending this framework to general probabilistic programs. 
%Every probabilistic program can be viewed as a pMC whose state space is the product of the program counter and variable valuations, which often induces infinite or continuous state spaces. 
Because our estimator operates entirely on sampled execution traces without ever materializing this global state space, it provides a natural and scalable foundation for optimizing probabilistic programs and their compositions. 
On the foundational side, we also aim to develop a high-level, structural understanding of this framework, characterizing path-based gradient estimation through abstract categorical and compositional semantics.
On the practical side, we plan to explore applications in machine learning, such as discrete latent-variable architectures (e.g., discrete VAEs) and recurrent stochastic models.
Furthermore, because sampled paths are independent, the algorithm is naturally amenable to massive GPU parallelization, which could drastically accelerate computation.

Unlike prevailing automatic differentiation techniques that depend on continuous relaxations or reparameterization tricks to handle discrete randomness, our score-function approach requires no reparameterization and natively accommodates stochastic branching (\texttt{if}-statements) and loops (\texttt{while}-statements).

\subsection*{AI Use Statement}
We used generative AI tools solely to provide minor support for code optimization and manuscript polishing. 
We did not use AI tools to develop theoretical models, formulate mathematical claims, write proofs, or design experimental methodology. 
All generated suggestions were verified by the authors, and we take full responsibility for the entirety of the text, code, and theoretical results presented in this work.

\subsection*{Acknowledgment}
This work was supported by JST SPRING, Japan Grant Number JPMJSP2104.

\bibliography{iclr2027_conference}

@InProceedings{Heck22,
author="Heck, Linus
and Spel, Jip
and Junges, Sebastian
and Moerman, Joshua
and Katoen, Joost-Pieter",
editor="Finkbeiner, Bernd
and Wies, Thomas",
title="Gradient-Descent for Randomized Controllers Under Partial Observability",
booktitle="Verification, Model Checking, and Abstract Interpretation",
year="2022",
publisher="Springer International Publishing",
address="Cham",
pages="127--150",
isbn="978-3-030-94583-1"
}

@article{williams1992simple,
  title={Simple statistical gradient-following algorithms for connectionist reinforcement learning},
  author={Williams, Ronald J},
  journal={Machine learning},
  volume={8},
  number={3},
  pages={229--256},
  year={1992},
  publisher={Springer}
}

@misc{mohamed2020montecarlogradientestimation,
      title={Monte Carlo Gradient Estimation in Machine Learning}, 
      author={Shakir Mohamed and Mihaela Rosca and Michael Figurnov and Andriy Mnih},
      year={2020},
      eprint={1906.10652},
      archivePrefix={arXiv},
      primaryClass={stat.ML},
      url={https://arxiv.org/abs/1906.10652}, 
}

@book{barthe2020foundations,
  editor    = {Gilles Barthe and Joost-Pieter Katoen and Alexandra Silva},
  title     = {Foundations of Probabilistic Programming},
  publisher = {Cambridge University Press},
  year      = {2020},
  doi       = {10.1017/9781108770750}
}

@article{vandemeent2018introduction,
  author  = {Jan-Willem van de Meent and Brooks Paige and Hongseok Yang and Frank Wood},
  title   = {An Introduction to Probabilistic Programming},
  journal = {arXiv preprint arXiv:1809.10756},
  year    = {2018}
}

@inproceedings{barthe2012probabilistic,
  author    = {Gilles Barthe and Boris K{\"o}pf and Federico Olmedo and Santiago Zanella-B{\'e}guelin},
  title     = {Probabilistic Relational Reasoning for Differential Privacy},
  booktitle = {Principles of Programming Languages (POPL)},
  pages     = {97--110},
  publisher = {ACM},
  year      = {2012},
  doi       = {10.1145/2103656.2103670}
}

@inproceedings{roy2021learning,
  author    = {Subhajit Roy and Justin Hsu and Aws Albarghouthi},
  title     = {Learning Differentially Private Mechanisms},
  booktitle = {IEEE Symposium on Security and Privacy (S\&P)},
  pages     = {852--865},
  publisher = {IEEE},
  year      = {2021},
  doi       = {10.1109/SP40001.2021.00060}
}

@inproceedings{carbin2013verifying,
  author    = {Michael Carbin and Sasa Misailovic and Martin C. Rinard},
  title     = {Verifying Quantitative Reliability for Programs That Execute on Unreliable Hardware},
  booktitle = {Object-Oriented Programming, Systems, Languages \& Applications (OOPSLA)},
  pages     = {33--52},
  publisher = {ACM},
  year      = {2013},
  doi       = {10.1145/2509136.2509546}
}

@inproceedings{wang2021central,
  author    = {Di Wang and Jan Hoffmann and Thomas Reps},
  title     = {Central Moment Analysis for Cost Accumulators in Probabilistic Programs},
  booktitle = {Programming Language Design and Implementation (PLDI)},
  pages     = {559--573},
  publisher = {ACM},
  year      = {2021},
  doi       = {10.1145/3453483.3454062}
}

@article{bingham2019pyro,
  author  = {Eli Bingham and Jonathan P. Chen and Martin Jankowiak and Fritz Obermeyer and Neeraj Pradhan and Theofanis Karaletsos and Rohit Singh and Paul Szerlip and Paul Horsfall and Noah D. Goodman},
  title   = {Pyro: Deep Universal Probabilistic Programming},
  journal = {Journal of Machine Learning Research},
  volume  = {20},
  number  = {28},
  pages   = {1--6},
  year    = {2019}
}

@article{lew2023adev,
  author  = {Alexander K. Lew and Mathieu Huot and Sam Staton and Vikash K. Mansinghka},
  title   = {{ADEV}: Sound Automatic Differentiation of Expected Values of Probabilistic Programs},
  journal = {Proceedings of the ACM on Programming Languages},
  volume  = {7},
  number  = {POPL},
  pages   = {121--153},
  year    = {2023},
  doi     = {10.1145/3571198}
}

@Book{		  BaierK08,
  author	= {Baier, Christel and Katoen, Joost-Pieter},
  title		= {Principles of Model Checking},
  year		= 2008,
  isbn		= {026202649X, 9780262026499},
  publisher	= {The MIT Press}
}

@article{junges2024parameter,
  author  = {Sebastian Junges and Erika {\'A}brah{\'a}m and Christian Hensel and Nils Jansen and Joost-Pieter Katoen and Tim Quatmann and Matthias Volk},
  title   = {Parameter Synthesis for {M}arkov Models: Covering the Parameter Space},
  journal = {Formal Methods in System Design},
  volume  = {62},
  number  = {1},
  pages   = {181--259},
  year    = {2024},
  doi     = {10.1007/s10703-023-00442-x}
}

@inproceedings{daws2004symbolic,
  author    = {Conrado Daws},
  title     = {Symbolic and Parametric Model Checking of Discrete-Time {M}arkov Chains},
  booktitle = {Theoretical Aspects of Computing (ICTAC 2004)},
  series    = {LNCS},
  volume    = {3407},
  pages     = {280--294},
  publisher = {Springer},
  year      = {2005},
  doi       = {10.1007/978-3-540-31862-0_21}
}

@inproceedings{dehnert2015prophesy,
  author    = {Christian Dehnert and Sebastian Junges and Nils Jansen and Florian Corzilius and Matthias Volk and Harold Bruintjes and Joost-Pieter Katoen and Erika {\'A}brah{\'a}m},
  title     = {{PROPhESY}: A {PRO}babilistic {P}aram{E}ter {SY}nthesis Tool},
  booktitle = {Computer Aided Verification (CAV)},
  series    = {LNCS},
  volume    = {9206},
  pages     = {214--231},
  publisher = {Springer},
  year      = {2015},
  doi       = {10.1007/978-3-319-21690-4_13}
}

@article{cubuktepe2022convex,
  author  = {Murat Cubuktepe and Nils Jansen and Sebastian Junges and Joost-Pieter Katoen and Ufuk Topcu},
  title   = {Convex Optimization for Parameter Synthesis in {MDP}s},
  journal = {IEEE Transactions on Automatic Control},
  volume  = {67},
  number  = {12},
  pages   = {6333--6348},
  year    = {2022},
  doi     = {10.1109/TAC.2021.3133265}
}

@inproceedings{badings2023efficient,
  author    = {Thom Badings and Sebastian Junges and Ahmadreza Marandi and Ufuk Topcu and Nils Jansen},
  title     = {Efficient Sensitivity Analysis for Parametric Robust {M}arkov Chains},
  booktitle = {Computer Aided Verification (CAV)},
  series    = {LNCS},
  volume    = {13966},
  pages     = {62--85},
  publisher = {Springer},
  year      = {2023},
  doi       = {10.1007/978-3-031-37709-9_4}
}

@inproceedings{junges2018finite,
  author    = {Sebastian Junges and Nils Jansen and Ralf Wimmer and Tim Quatmann and Leonore Winterer and Joost-Pieter Katoen and Bernd Becker},
  title     = {Finite-State Controllers of {POMDP}s using Parameter Synthesis},
  booktitle = {Uncertainty in Artificial Intelligence (UAI)},
  pages     = {519--529},
  publisher = {AUAI Press},
  year      = {2018}
}

@inproceedings{polgreen2016data,
  author    = {Elizabeth Polgreen and Viraj B. Wijesuriya and Sofie Haesaert and Alessandro Abate},
  title     = {Data-Efficient {B}ayesian Verification of Parametric {M}arkov Chains},
  booktitle = {Quantitative Evaluation of Systems (QEST)},
  series    = {LNCS},
  volume    = {9826},
  pages     = {35--51},
  publisher = {Springer},
  year      = {2016},
  doi       = {10.1007/978-3-319-43425-4_3}
}

@inproceedings{volk2016advancing,
  author    = {Matthias Volk and Sebastian Junges and Joost-Pieter Katoen},
  title     = {Advancing Dynamic Fault Tree Analysis -- Get Succinct State Spaces Fast and Synthesise Failure Rates},
  booktitle = {Computer Safety, Reliability, and Security (SAFECOMP)},
  series    = {LNCS},
  volume    = {9922},
  pages     = {253--265},
  publisher = {Springer},
  year      = {2016},
  doi       = {10.1007/978-3-319-45477-1_20}
}

@inproceedings{bartocci2011model,
  author    = {Ezio Bartocci and Radu Grosu and Panagiotis Katsaros and C. R. Ramakrishnan and Scott A. Smolka},
  title     = {Model Repair for Probabilistic Systems},
  booktitle = {Tools and Algorithms for the Construction and Analysis of Systems (TACAS)},
  series    = {LNCS},
  volume    = {6605},
  pages     = {326--340},
  publisher = {Springer},
  year      = {2011},
  doi       = {10.1007/978-3-642-19835-9_30}
}

@article{ceska2017precise,
  author  = {Milan {\v{C}}e{\v{s}}ka and Frits Dannenberg and Nicola Paoletti and Marta Kwiatkowska and Lubo{\v{s}} Brim},
  title   = {Precise Parameter Synthesis for Stochastic Biochemical Systems},
  journal = {Acta Informatica},
  volume  = {54},
  number  = {6},
  pages   = {589--623},
  year    = {2017},
  doi     = {10.1007/s00236-016-0265-2}
}

@book{davis2006direct,
  author    = {Timothy A. Davis},
  title     = {Direct Methods for Sparse Linear Systems},
  series    = {Fundamentals of Algorithms},
  publisher = {SIAM},
  year      = {2006},
  doi       = {10.1137/1.9780898718881}
}

@book{saad2003iterative,
  author    = {Yousef Saad},
  title     = {Iterative Methods for Sparse Linear Systems},
  edition   = {2nd},
  publisher = {SIAM},
  year      = {2003},
  doi       = {10.1137/1.9780898718003}
}

@inproceedings{spel2021finding,
  author    = {Jip Spel and Sebastian Junges and Joost-Pieter Katoen},
  title     = {Finding Provably Optimal {M}arkov Chains},
  booktitle = {Tools and Algorithms for the Construction and Analysis of Systems (TACAS)},
  series    = {LNCS},
  volume    = {12651},
  pages     = {173--190},
  publisher = {Springer},
  year      = {2021},
  doi       = {10.1007/978-3-030-72016-2_10}
}

@inproceedings{heck2025generalized,
  author    = {Linus Heck and Tim Quatmann and Jip Spel and Joost-Pieter Katoen and Sebastian Junges},
  title     = {Generalized Parameter Lifting: Finer Abstractions for Parametric {M}arkov Chains},
  booktitle = {Automated Technology for Verification and Analysis (ATVA)},
  series    = {LNCS},
  pages     = {207--230},
  publisher = {Springer},
  year      = {2025},
  doi       = {10.1007/978-3-032-08707-2_10}
}

@ARTICLE{MarbachT01,
  author={Marbach, P. and Tsitsiklis, J.N.},
  journal={IEEE Transactions on Automatic Control}, 
  title={Simulation-based optimization of Markov reward processes}, 
  year={2001},
  volume={46},
  number={2},
  pages={191-209},
  doi={10.1109/9.905687}}

@article{BaxterB01,
  title={Infinite-Horizon Policy-Gradient Estimation},
  volume={15},
  ISSN={1076-9757},
  url={http://dx.doi.org/10.1613/jair.806},
  DOI={10.1613/jair.806},
  journal={Journal of Artificial Intelligence Research},
  publisher={AI Access Foundation},
  author={Baxter, J. and Bartlett, P. L.},
  year={2001},
  month=Nov, pages={319–350} }
\bibliographystyle{iclr2027_conference}

\appendix
\section{Benchmark Generation}\label{appendix:benchmarkGeneration}
We developed a scalable benchmark generator for pMCs.
The generation proceeds in two steps:
\begin{enumerate}[nosep, leftmargin=1.5em]
  \item \textbf{Topology Generation:} 
    We construct a finite, sparse directed graph ensuring that every state is \emph{accessible} from the initial state 
    and \emph{co-accessible} to at least one output state. 
    To scale up to large state spaces ($>10^5$ states), with non-trivial cyclic behavior, 
    we generate modular sparse subgraphs with cycles, duplicate them, 
    and link them with directed bridging transitions to guarantee global connectivity.
  \item \textbf{Parametric Transitions:} 
    For every state $s$, outgoing transition probabilities $\{P_{s, s'}\}_{s'}$ are parameterized by affine (degree-$1$) univariate polynomials over a parameter $X_i \in \{X_1, \dots, X_k\}$.
    These polynomials are constructed such that their formal sum satisfies $\sum_{s'} P_{s, s'}(X_i) = 1$.
    For a designated initial parameter point $p_0$, we ensure $P_{s, s'}(p_0) \in (0, 1)$ for all transitions. 
    By continuity, there exists an open neighborhood $B(p_0, \varepsilon)$ where all transition probabilities remain strictly in $(0, 1)$, ensuring the pMC remains well-defined and differentiable under local gradient updates.
\end{enumerate}
Although individual transitions use degree-$1$ polynomials over a single parameter, paths through cycles and branches combine these transitions into high-degree multivariate rational functions for reachability probabilities and expected rewards.
Our generator is implemented in Python and scales to models with $10^5$ states and millions of transitions in seconds. 
All generated models are exported as \textsc{Prism} files, ensuring that all tools evaluate the exact same benchmarks.

\end{document}